\documentclass{article}
\usepackage{actilabel}

\usepackage{times}
\usepackage{soul}
\usepackage{url}
\usepackage[hidelinks]{hyperref}
\usepackage[utf8]{inputenc}
\usepackage[small]{caption}
\usepackage{graphicx}
\usepackage{amsmath}
\usepackage{amsthm}
\usepackage{booktabs}
\usepackage{algorithm}
\usepackage{algpseudocode}
\usepackage{multirow}
\usepackage{subcaption}
\usepackage{booktabs}
\usepackage{xcolor}
\newtheorem{definition}{\textbf{Definition}}
\newtheorem{problem}{\textbf{Problem}}
\newtheorem{thm}{\textbf{Theorem}}
\newtheorem{lemma}{\textbf{Lemma}}
\newcommand{\figref}[1]{Figure~\ref{fig:#1}}
\newcommand{\tblref}[1]{Table~\ref{tbl:#1}}

\newcommand{\probref}[1]{Problem~\ref{prob:#1}}
\newcommand{\lemmaref}[1]{Lemma~\ref{lemma:#1}}

\title{ActiLabel: A Combinatorial Transfer Learning Framework for Activity Recognition}

\author{
Parastoo Alinia$^1$
\and
Iman Mirzadeh$^1$\and
Hassan Ghasemzadeh$^1$
\affiliations
$^1$Washington State University,  Pullman, WA, United States\\
\emails
\{parastoo.alinia, hassan.ghasemzadeh, seyediman.mirzadeh\}@wsu.edu
\vspace{10mm}
}

\begin{document}

\maketitle
\begin{abstract}
Sensor-based human activity recognition has become a critical component of many emerging applications ranging from behavioral medicine to gaming. However, an unprecedented increase in the diversity of sensor devices in the Internet-of-Things era has limited the adoption of activity recognition models for use across different domains. We propose {\it ActiLabel}, a combinatorial framework that learns structural similarities among the events in an arbitrary domain and those of a different domain. The structural similarities are captured through a graph model, referred to as the {\it dependency graph}, which abstracts details of activity patterns in low-level signal and feature space. The activity labels are then autonomously learned by finding an optimal tiered mapping between the dependency graphs. Extensive experiments based on three public datasets demonstrate the superiority of ActiLabel over state-of-the-art transfer learning and deep learning methods.

\end{abstract}

\section{Introduction}

Human activity recognition (HAR) systems are crucial components in health monitoring and personalized behavioral medicine. HAR systems use machine learning algorithms to detect physical activities based on the data collected from wearable and mobile sensors \cite{piwek2016rise,pantelopoulos2010survey}. Such systems are usually designed based on labeled training data collected in a particular domain, such as with a specific sensor modality, wearing site, or user. A significant challenge with existing HAR systems is that the baseline machine learning model which is trained with a specific setting (i.e., source) performs poorly in new settings \cite{zhang2012generalization,Wang2018}. This challenge has limited scalability of sensor-based HAR system given collecting sufficiently large amounts of labeled sensor data for every possible domain is a time-consuming, labor-intensive, and often infeasible process.

We introduce {\it ActiLabel}, a combinatorial framework that learns machine learning models in a new domain (i.e., target) without the need to manually collect any labels. A unique attribute of ActiLabel is that it examines structural relationships between activity events (i.e., classes/clusters) in two different domains and uses this information for target-to-source mapping. Such structural relationships allow us to compare the two domains at a higher level of abstraction than the common feature space, therefore enable knowledge transfer across radically diverse domains. We hypothesize that even under sever cross-domain spatial and temporal uncertainties (i.e., significant distribution shift), physical activities exhibit similar structural dependencies across different domains, mainly due to the physical and physiological underpinning of human health monitoring. 

To the best of our knowledge, our work is the first study that develops a combinatorial approach for structural transfer learning. Our notable contributions can be summarized as follows. (i) We introduce a combinatorial optimization formulation for transfer learning; (ii) we devise methodologies for constructing a network representation of wearable sensor readings, referred to as {\it network graph}; (iii) we design algorithms that perform community detection on the network graph to identify core activity clusters; (iv) we introduce an approach to construct a dependency graph based on the core activity clusters identified on the network graph; (vi) we propose a novel multi-layer matching algorithm for mapping target-to-source dependency graphs; (vii) we conduct an extensive assessment of the performance of ActiLabel for cross-modality, cross-subject, and cross-location activity learning using real sensor data collected with human subjects.

\section{Background and Related Work}
\subsection{Transfer Learning}
Transfer learning (TL) is the ability to extend the knowledge in one setting to another, nonidentical but related, setting. We refer to the previous setting as the {\it source domain}. The sensor data captured in this domain is referred to as the source dataset, which is fully labeled in our case. The new state of the system, which may exhibit radical changes from the source domain, is referred to as the {\it target domain}, where we intend to label the sensor data autonomously \cite{cook2013transfer,pan2010survey}. Depending on how the availability of the labels in the source and target, one can categorize TL techniques into three groups. Inductive TL is where the source is fully labeled and there are few labeled samples in the target. In transductive TL, which is the focus of this paper, labels are available in the source, but there is no label in the target. Unsupervised TL is where there is no label in neither target or source domains \cite{weiss2016survey,redko2016theoretical}. 
Prior research also proposed a deep convolution recurrent neural network to automate the process of feature extraction and to capture general patterns from activity data \cite{ordonez2016deep}. However, deep learning models have not shown promising performance in highly diverse domains, such as cross-modality knowledge transfer. For example, previous research achieved only $54.2$\% accuracy in recognizing human gestures using deep learning with computationally dense algorithms cross sensors of different modalities \cite{zhu2017multimodal,feichtenhofer2016convolutional}. More advanced models combine knowledge of transfer and deep learning \cite{Yang2017}. There have been studies that transfer different layers of deep neural networks across different domains. In one study, a cross-domain deep transfer learning method was introduced that achieved $64.6$\% accuracy with four activity classes for cross-location and cross-subject knowledge transfer \cite{Wang2018}. Unlike our transductive transfer learning approach in this paper, these approaches fall within the category of inductive transfer learning, where some labeled instances are required in the target domain.

\subsection{Graph Theory Definitions}
k-Nearest Neighbor (k-NN) graphs are commonly used to classify unknown events using feature representations. During the classification process, certain features are extracted from unknown events and classified based on the features extracted from their k-nearest neighbors \cite{chen2009fast,maier2009influence}. k-NN graph of a dataset is obtained by connecting each data point to its k closest points from the dataset based on a distance metric between the data points. The symmetric k-NN graphs are when each point is connected to another only if both are in each other k-nearest neighborhood. 

Community detection algorithms are widely used to identify clusters in large scale network graphs \cite{ferreira2016time}. Recent research suggests that detecting communities from a network representation of data could result in a higher clustering performance compared to traditional clustering algorithms \cite{puxeddu2017community,blondel2008fast}. We define some of the essential features related to community detection in network graphs in the following.

 \begin{definition}[\textbf{Cut}]
    Given a graph $G$($V_N$,$E_N$) and communities $\mathcal{C}$ = \{$C_1$, $\dots$, $C_K$\}, ''{\it Cut}'' between communities $C_i$ and $C_j$ is defined as the number of edges $(u,v)$ with one end in $C_i$ and the other end in $C_j$. That is,
    
    \begin{equation}
        Cut(C_i,C_j) = |(u,v) \in E_N : u \in C_i ~~\&~~ v \in C_j| 
    \end{equation}
    \end{definition}

    \begin{definition}[\textbf{Cluster Density}]
    Given a graph $G$($V_N$,$E_N$) and communities $\mathcal{C}$ = \{$C_1$, $\dots$, $C_K$\} within the graph $G$, ''{\it community density}'', $\Delta$($C_i$), for community $C_i$ is defined as the number of edges $(u,v)$ with both ends residing in $C_i$.
    
    \begin{equation}
        \Delta(C_i) = |(u,v) \in E_N : u \in C_i ~~\&~~ v \in C_i| 
    \end{equation}
    \end{definition}

    \begin{definition}[\textbf{Community Size}]
    Given a graph $G$($V_N$,$E_N$) and communities $\mathcal{C}$ = \{$C_1$, $\dots$, $C_K$\} within the graph $G$, ''{\it Community Size}'', $\sigma$($C_i$), for community $C_i$ is defined as the number of vertices that reside in $C_i$.
    
    \begin{equation}
        \sigma(C_i) = |v \in V_N : v \in C_i| 
    \end{equation}
    \end{definition}

\section{ActiLabel}
We propose ActiLabel to solve the problem of labeling sensor observations in an arbitrary setting (i.e., target) based on the labeled observations in another setting (i.e., source) even when the source and target observations follow highly diverse distributions. ActiLabel aims to create a labeled dataset in the target by transferring the knowledge from the labeled source observations such that the labeling error is minimized. 
 

Assigning a label to each sensor observation in the target domain can be viewed as a mapping problem where sensor observations in the target domain are mapped to labeled observations in the source domain. ActiLabel finds an optimal mapping between the two domain; the mapping, however, is performed at a much higher level of abstraction that the traditional feature level. To this end, mapping in ActiLabel is done from groups of similar target observations, called \textit{core clusters}, to known activity classes in the source domain. The goal of this optimization problem is to minimize the mapping costs/error.


The overall approach in ActiLabel is illustrated in \figref{main-process}. As summarized in Algorithm~\ref{alg:ActiLabel}, the design process in ActiLabel involves the following steps, where we refer to the first two steps as {\it graph modeling} and the next two steps as {\it optimal label learning}. (i) Network graph construction from sensor readings in both domains  \figref{main-process}-a; (ii) Core cluster identification given the network graphs in both domains \figref{main-process}-b. (iii) Dependency graph construction based on the core clusters and network graph in both domains \figref{main-process}-c. (iv) Optimal Label learning by mapping the dependency graphs from the source and target domains  \figref{main-process}-d, \figref{main-process}-e, and \figref{main-process}-f. 

\begin{figure*}[tbh!]
		\centering
		\includegraphics[width=0.8\linewidth]{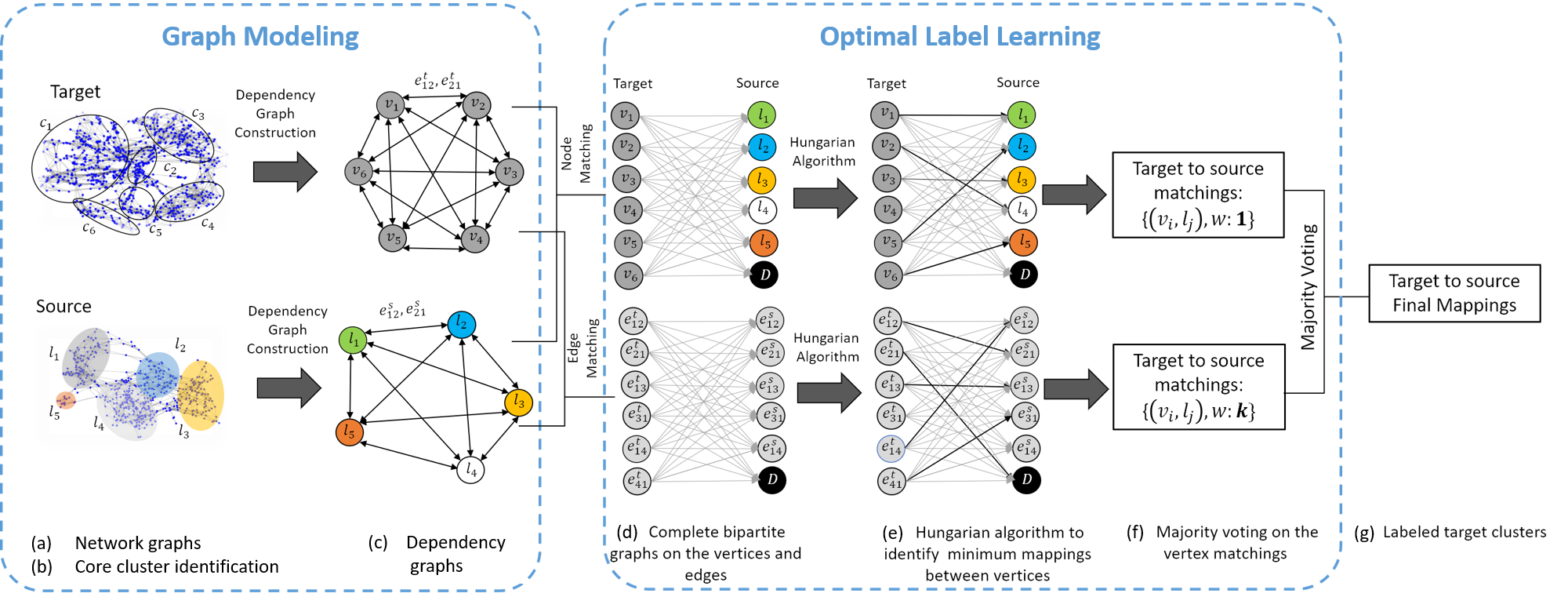}
		\caption{An overview of the ActiLabel framework including graph modeling and optimal label learning.}
		\label{fig:main-process}
\end{figure*}

\begin{algorithm}
\small
		\caption{ActiLabel}
		\begin{algorithmic}[1]
		\Statex \textbf{Input:}\textit{$D_t$, unlabeled target dataset, \{$D_s$, $L_s$\}, labeled source dataset.}
         \Statex \textbf{Result:} {\it Labeled target dataset, \{$D_t$, $L_t$\}}
        \Statex \textbf{Graph Modeling:} \Comment{(section~\ref{sec:meccm})}
        
		 \State \qquad \textit{Construct Network Graphs in both domains;} \Comment{(section~\ref{subsec:network-graph-construct}) }
		  \State \qquad \textit{Identify core clusters in both domains; } \Comment{(section~\ref{subsec:core-cluster-identification})}
		  \State \qquad \textit{Build Dependency graphs;} \Comment{(section~\ref{subsec:dependency-graph-construct})}
		  \State \qquad \textit{Extract structural relationships among the core clusters in both domains;}
	   
	        \Statex \textbf{Optimal Label Learning} \Comment{(section~\ref{subsec:optimal-label-learning})}
			\State \qquad \textit{Perform graph-level min-cost mapping from target to source;}
			\State \qquad \textit{Assign labels to the observations in target;}
			\State \qquad \textit{Train activity recognition model in target using new labels;}
		\end{algorithmic}
		\label{alg:ActiLabel}
	\end{algorithm}

	\subsection{Graph Modeling} \label{sec:meccm}
   We construct dependency graphs that capture structural dependencies among the events (i.e., physical activities) in both target and source domains. The dependency graphs are then used in {\it optimal label learning} to label sensor observations and generate a training dataset in the target domain. As shown in \figref{main-process}, our graph modeling consists of three phases: (i) network graph construction; (ii) core cluster identification; and (iii) dependency graph construction. This section elaborates on each phase.
    
    \subsubsection{Network Graph Construction}
    \label{subsec:network-graph-construct}
    We initially build a network representation of the sensor observations based on symmetric k-nearest-neighboring to quantify the amount of similarity between pairs of observations.
    
    \begin{definition}[Network Graph]
        The network graph refers $G_N$($V_N$,$E_N$) is a symmetric k-NN graph where vertices are feature representation of the sensor data and distance function is the cosine similarity between the features.
    \end{definition}

    \subsubsection{Core Cluster Identification}
    \label{subsec:core-cluster-identification}

    To identify core clusters in ActiLabel, we propose a graph-based clustering algorithm similar to the approach in prior research \cite{barton2019chameleon}. We refer to this approach as {\it core cluster identification} (CCI), which runs on the network graph $G$($V_N$,$E_N$) in two steps. (i) Partitioning the network graph into multiple communities of approximately the same vertex size using a greedy community detection technique. (ii) Merging the communities with the highest similarity score based on their dendrogram structure. 
    
    The amount of similarity $\alpha_{i,j}$ between communities $C_i$ and $C_j$ is measured as the ratio of the number of edges between the two communities (i.e., $Cut$($C_i$,$C_j$)) to the average number of edges that reside within the two communities. Therefore, the similarity score of $\alpha_{i,j}$ is given by
    
    \begin{equation}\label{eq:similarity}
    \alpha(i,j)= \frac{Cut(C_i,C_j)}{\frac{|C_i|+|C_j|}{2}}
    \end{equation}
    
    \noindent where the terms $|C_i|$ and $|C_j|$ denote the number of edges that reside in $C_i$ and $C_j$, respectively. Note that the similarity score $\alpha$ is defined such that it is not adversely influenced by the size of communities in unbalanced datasets.

	\subsubsection{Dependency Graph Construction}
	\label{subsec:dependency-graph-construct}
	To capture high-level structural relationships among sensor observations, we devise a structural dependency graph where the core clusters identified previously represent vertices of the dependency graph.
	
	\begin{definition}[Dependency Graph]
		Given a network graph $G$($V_N$,$E_N$) where $|V_N|$ = $|\mathcal{X}|$ and core clusters $\mathcal{C}$ = \{$C_1$, $\dots$, $C_K$\} obtained from the network graph, we define dependency graph $G$($V_D$ ,$E_D$, $W^v_D$, $W^e_D$) as a weighted directed complete graph as follows. Each vertex $u_i \ in V_D$ is associated with a core cluster $C_i \in \mathcal{C}$. Thus, $|V_D|$ = $|\mathcal{C}|$. Each vertex $u_i \in V_D$ is assigned a weight $w^u_i$ given by
		
		\begin{equation}\label{eq:vertex_weight}
		 w^u_i  = \frac{\Delta(C_i)}{\sigma(C_i)|}
		\end{equation}
		
		\noindent where $\Delta(C_i)$ and $\sigma(C_i)$ refer to cluster density and cluster size, respectively, for core cluster $C_i$. Each edge $(u_i, u_j) \in E_D$, associated with core clusters $C_i$ and $C_j$, is assigned a weight $w^e_{ij}$ given by
		
        \begin{equation}\label{eq:edge_weight}
            w^e_{ij} = \frac{Cut(C_i, C_j)}{\sigma(C_j)}
        \end{equation}
		\end{definition}


   	\begin{algorithm}[tbh]
   	\small
        \caption{Optimal Label Learning}
        \begin{algorithmic}[1]
            \Statex \textbf{Input:}\textit{$G_D^t$ and $G_D^s$, dependency graphs for target and source domains.}
            \Statex \textbf{Result:} {\it Labeled target dataset, \{$D_t$, $L_t$\}}
            \State {\it Construct bipartite graph $BG_e$ using edge components;}
             \State {\it Obtain bipartite mapping $M_e$ on $GB_e$;}
            \State {\it  Construct bipartite graph $BG_v$ on vertex components;}
            \State {\it  Obtain bipartite mapping $M_v$ on $GB_v$;}
            \State {\it  Construct bipartite graph $BG_c$ using $M_e$ and $M_v$;}
            \State {\it Obtain bipartite mapping OptMapping on $GB_c$;}
            \State {\it Assign source labels to appropriate core clusters in target using OptMapping;}
        \end{algorithmic}
        \label{depmatch}
        
    \end{algorithm}
    
    \subsection{Optimal Label Learning}
    \label{subsec:optimal-label-learning}
    Algorithm \ref{depmatch} summarizes the steps for optimal label learning. The goal of the optimal label learning is to find a mapping from the dependency graph in the target to that of the source domain while minimizing the mapping error. We refer to this optimization problem as min-cost dependency graph mapping and define it as follows.

    \begin{problem}[Min-Cost Dependency Graph Mapping]\label{prob:dgmcm}
        Let $G_D^s$ and $G_D^t$ denote dependency graphs obtained from datasets in the source and target domains, respectively. The min-cost dependency graph mapping is to find a mapping $R : G_D^t \rightarrow G_D^s$ from $G_D^t$ to $G_D^s$ such as the cost of such mapping is minimized.
    \end{problem}

    \probref{dgmcm} can be viewed as a combinatorial optimization problem that finds an optimal mapping in a two-tier fashion: (i) it initially performs component-level mappings where vertex-wise and edge-wise mappings are found between source and target dependency graphs; and (ii) it then uses the component-level mappings to reach a consensus about the optimal mapping for the problem as a whole. Such a two-level mapping problem can be represented using the objective in \eqref{eq:graphmatching}. 
    \begin{equation}
    \label{eq:graphmatching}\
    Minimize~~~ \sum_{i=1}^{|V_D^t|}\sum_{j=1}^{|V_D^s|} 1-\frac{\mu(i,j)}{M}
    \end{equation}
    
    \noindent where $\mu(i,j)$ represents the number of mappings between $v_i \in V_D^t$ and $v_j \in V_D^s$ obtained through the component-level optimization. Furthermore, $M$ is a normalization factor that is equal to the total number of component-wise mappings. The objective in \eqref{eq:graphmatching} attempts to minimize the mapping cost at the graph-level and, therefore, can be viewed as the objective for \probref{dgmcm}.
    
    We build a weighted complete bipartite graph on the components of the dependency graphs to find the minimum double-cost mapping. \figref{main-process}-d is an illustration of such a bipartite graph where the nodes on the left shore of the graph represent components (e.g., node weights) of the target dependency graph and the nodes on the right shore of the bipartite graph are associated with corresponding components in the source dependency graph.
    
    In constructing a bipartite graph, a weight $\omega_{ij}$ is assigned to the edge that connects node $i$ in the target side to nodes $j$ in the source side. This weight also represents the actual mapping cost and is given by

    \begin{equation}
    \omega_{ij} = |w_{si}-w_{tj}|
    \end{equation}
    
    \noindent where $w_{si}$ and $w{tj}$ are the weight values associated with component $i$ in the source domain and component $j$ in the target domain, respectively. We note that these weights can be computed using \eqref{eq:vertex_weight} and \eqref{eq:edge_weight} for vertex-wise mapping and edge-wise mapping, respectively. We also note that if the number of components in source and target were not equal, we could add dummy nodes to one shore of the bipartite graph to create a complete and balanced bipartite graph. 
    
    We use Hungarian Algorithm (a widely used weighted bipartite matching algorithm with $O(m^3)$ time complexity) \cite{kuhn1955hungarian} to identify an optimal mapping from the nodes on the left shore of the bipartite graph to the nodes on the right shore of the graph. 

    The last step is to assign the labels mapped to each cluster to the target observations within that cluster. A classification model is trained on the labeled target dataset for physical activity recognition.
    
    \subsection{Time Complexity Analysis}
\begin{lemma}\label{lemma:tc_gm}
    The graph modeling in ActiLabel has a time complexity of $O(n^2)$ where '$n$' denotes the number of sensor observations.
\end{lemma}
\begin{proof}
    The proof is eliminated for brevity.
\end{proof}

\begin{lemma}\label{lemma:tc_oll}
    The optimal label learning phase in ActiLabel has a time complexity of $O(n + m^3)$ where $n$ denotes the number of sensor observations and $m$ represents the number of classes.
\end{lemma}
\begin{proof}
    The proof is eliminated for brevity.
\end{proof}

\begin{thm}\label{thm:complexity}
    The time complexity of ActiLabel is quadratic in the number of sensor observations, $n$.

\end{thm}
\begin{proof}
    Assuming that the number of classes, $m$, is much smaller than the number of sensor observations, $n$, (i.e., $m  \ll n$), the proof follows \lemmaref{tc_gm} and \lemmaref{tc_oll}.
\end{proof}

\section{Experimental Setup}
	\subsection{Datasets}
	We use three sizeable human activity datasets to evaluate the performance of ActiLabel. We refer to these datasets as PAMAP2, a physical activity monitoring dataset used in \cite{reiss2012introducing}, DAS, daily \& sport activity dataset used in \cite{barshan2014recognizing}, and Smartsock, a dataset containing ankle-worn sensor data used in \cite{fallahzadeh2016Smartsock}. \tblref{datasets} has provided a summary of the datasets utilized in this study.
	
	\begin{table}[ht]
	    \caption{Brief description of the datasets utilized for activity recognition. The sensor modalities include accelerometer: ACC, gyroscope: GYR, magnetometer: MAG, temperature: TMP, orientation: ORI, heart rate: HR, stretch sensor: STR, and locations are chest: C, ankle: A, hand: H, left arm: LA, left leg: LL, right arm: RA, right leg: RL, torso: T.  }
		\label{tbl:datasets}
		\centering
		\small
		\resizebox{0.45\textwidth}{!}{
		\begin{tabular}{lrrrrll}
			\toprule
			Dataset    & \#Sub. & \#Act. & \#Sample & Sensors  & Locations                                                                \\ \midrule
			PAMAP2     & 9         & 24         & 3850505        & \begin{tabular}[l]{@{}l@{}}ACC, GYR, \\ HR, TMP, \\ ORI, MAG\end{tabular} & C, H, A\\\midrule
			DAS        & 8        & 19          & 1140000 & \begin{tabular}[l]{@{}l@{}}ACC, GYR\\ MAG\end{tabular}                 &\begin{tabular}[l]{@{}l@{}} LA, RA, \\ LL, RL, T\end{tabular} \\ \midrule
			Smartsock & 12        & 12         &   9888      & ACC, STR              & A                              \\ \bottomrule
		\end{tabular}
		}
	\end{table}
	
	\begin{figure}[ht]
    \centering
        \includegraphics[width=\linewidth]{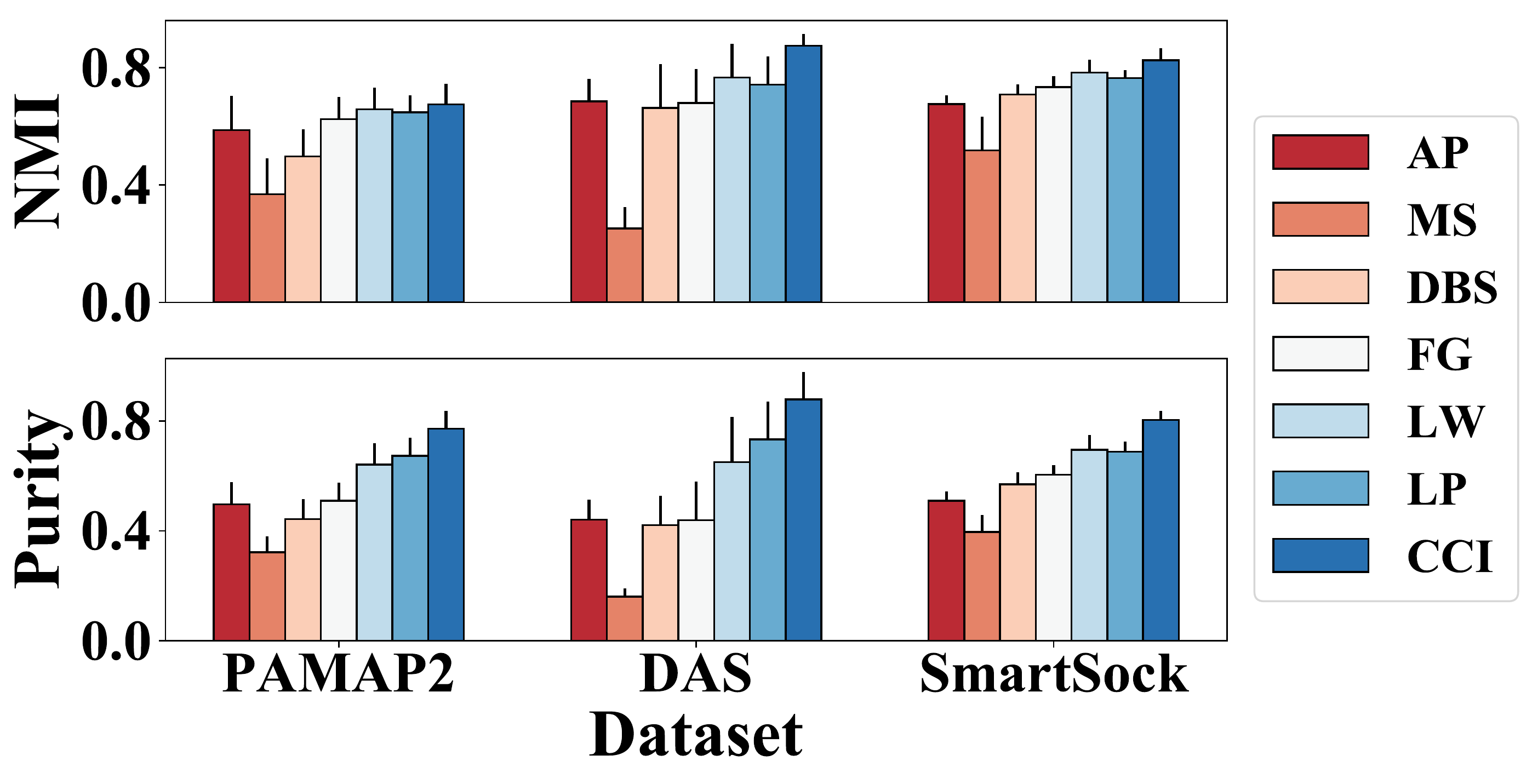}

      \caption{Performance comparison between core cluster identification in ActiLabel and standard clustering and communication detection algorithms.}
      \label{fig:nmi-avg}
      \vspace{-3mm}
    \end{figure}
    
    \subsection{Comparison Methods}
    We compare the performance of ActiLabel with the following algorithms. (i) Baseline, which learns a shallow classifier in the source domain and deploys it for activity recognition in the target domain. (ii) Deep Convolution LSTM, \cite{ordonez2016deep} which learns a deep classifier in the source domain and applies it for activity recognition in the target domain. (iii) DirectMap, which directly maps centroids of the clusters in the target to activity classes in the source domain to create a labeled dataset in the target. (iv) Upper-bound, which learns a classifier assuming that the actual labels are available in the target domain. We assess the performance of ActiLabel during three scenarios: (i) cross-modality transfer when sensors in the two domains have different modalities (e.g., accelerometer and gyroscope), (ii) cross-subject transfer across two different human subjects, and (iii) cross-location transfer when the target and source location of the wearable sensor is different. 
    
    \subsection{Implementation Details}
	The datasets are divided into $50\%$ training, $25\%$ test, and $25\%$ validation parts with no overlap to avoid bias. We extracted an exhaustive set of time-domain features from a sliding window of size 2 seconds with 25\% overlap. The extracted features include mean value, peak amplitude, entropy, and energy of the signal which are shown to be useful in human physical activity estimation using inertial sensor data \cite{mannini2010machine,saeedi2014cost}. We reduce the features dimension using UMAP \cite{mcinnes2018umap} algorithm before clustering. 
	
	We analyzed the effect of hyper-parameter $k$ in the $k$-NN network graph on the performance of CCI as measured by NMI and purity. As shown in \figref{k-NN}, NMI achieved its highest value (i.e., $0.67$ for PAMAP2, $0.88$ for DAS, and $0.83$ for Smartsock) when $k$ was set to $2$\% or $5$\% of the graph network size. This translates into a $k$=$8$ for PAMAP2 and Smartsock and $k=11$ for DAS datasets.
	 \begin{figure}[tbh!]
    \centering
        \begin{subfigure}{.24\textwidth}
          \includegraphics[width=\linewidth]{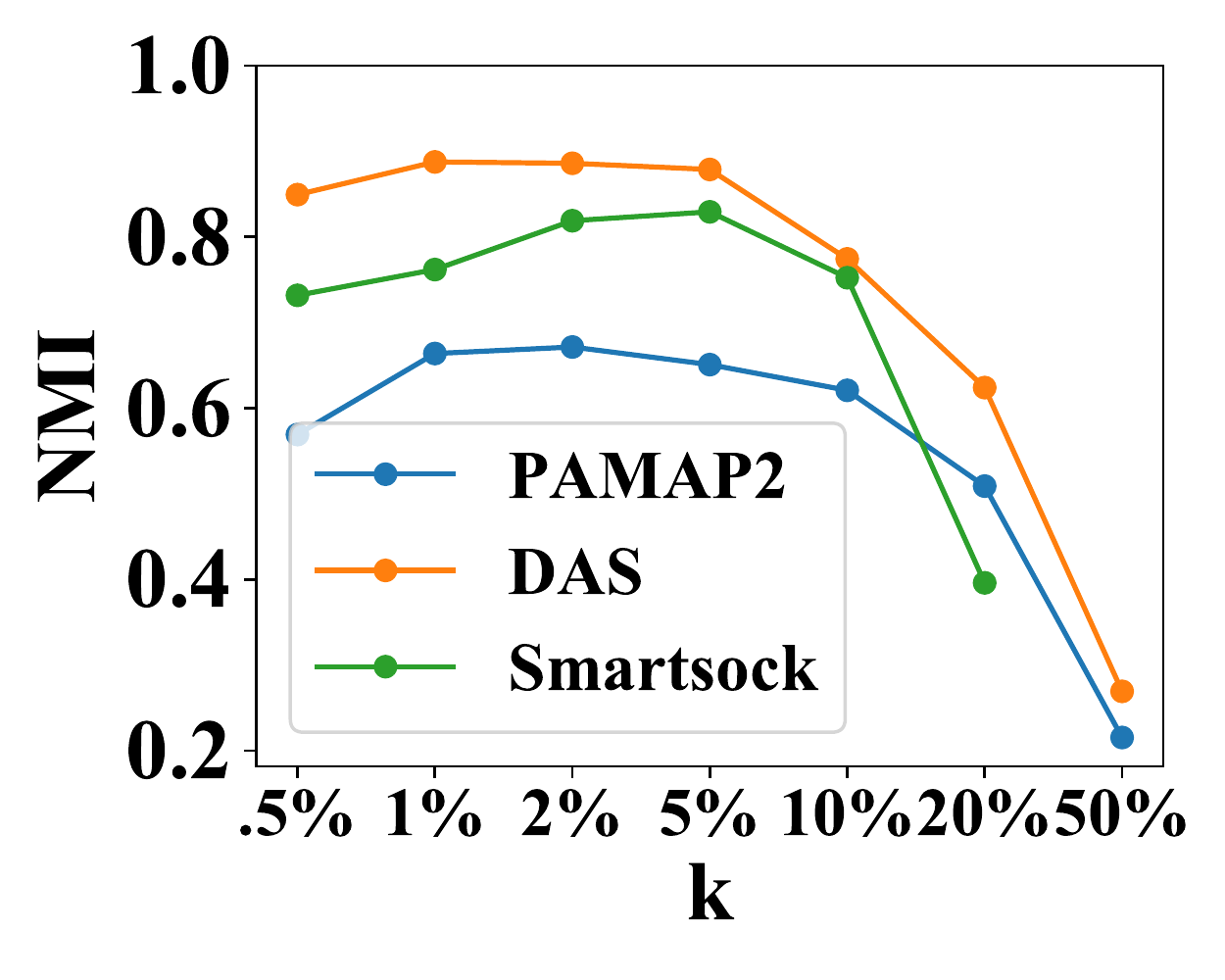}
          \caption{NMI}
          \label{fig:kfig1}
        \end{subfigure}%
        \begin{subfigure}{.24\textwidth}
          \includegraphics[width=\linewidth]{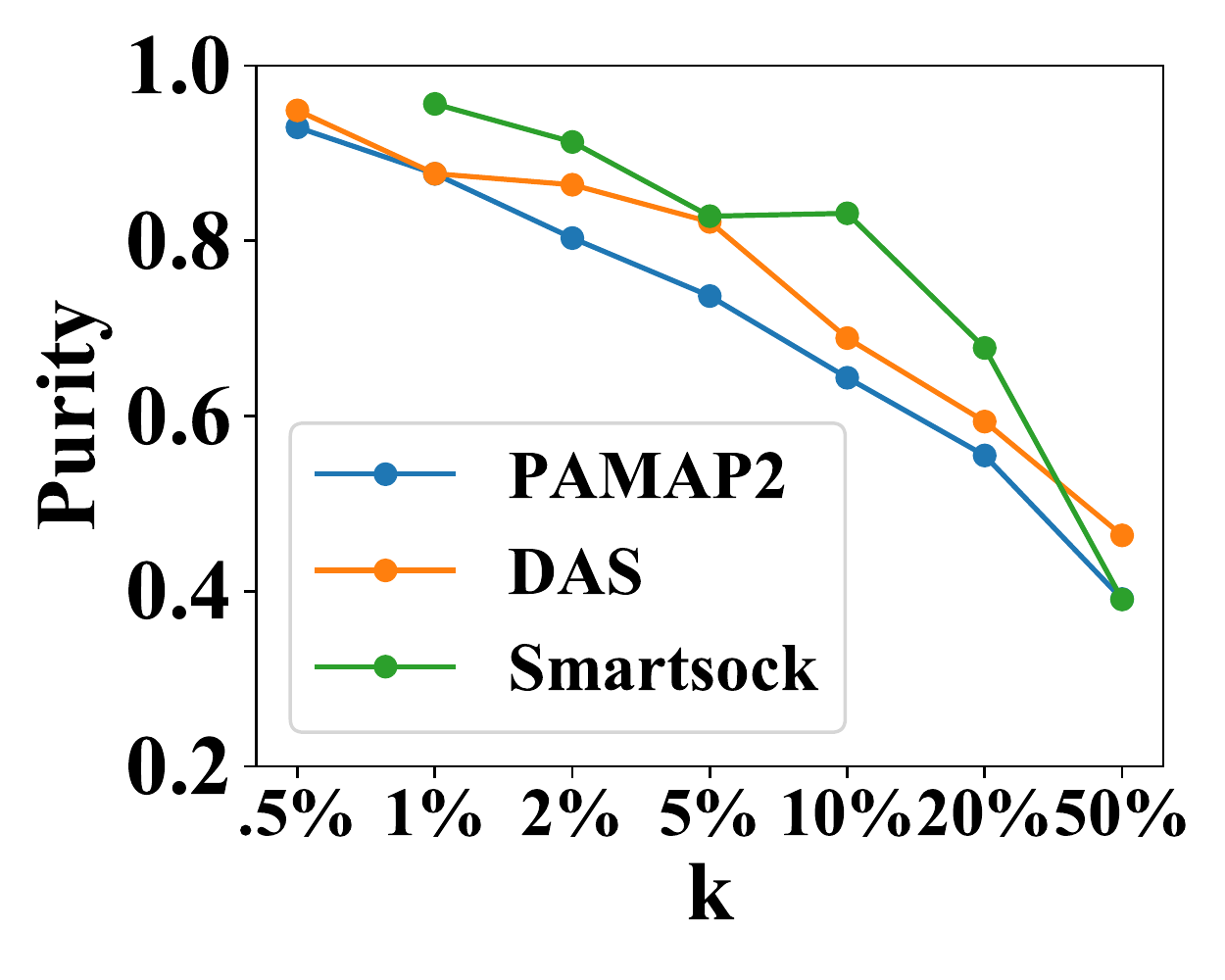}
          \caption{Purity}
          \label{fig:kfig2}
        \end{subfigure}
        \caption{Performance of CCI versus parameter $k$ in network graph construction.}
        \label{fig:k-NN}
    \end{figure}

    \section{Results}
    
    We analyzed the effect of hyper-parameter $k$ in the $k$-NN network graph on the performance of the core cluster identification as measured by {\it normalized mutual information (NMI)} and clustering {\it purity} \cite{rendon2011comparison}. The results demonstrate $k=8$ for PAMAP2 and Smartsock and $=k11$ for DAS datasets as an optimal value.

    \subsection{Performance of Core Cluster Identification}\label{sec:nmi}
    
    As shown in \figref{nmi-avg}, CCI outperforms state-of-the-art clustering and community detection algorithms. The NMI for the competing methods ranged from 0.37--0.65 for PAMAP2, 0.25--0.77 for DAS, and 0.52--0.76 for Smartsock. The proposed algorithm CCI increased NMI to 0.67, 0.87, and 0.85 for PAMAP2, DAS, and Smartsock datasets, respectively.  Note that the clustering was generally more accurate for Smartsock and DAS datasets because PAMAP2 contained data from sensor modalities (e.g., temperature) that might not be a good representative of the activities of interest.

    \subsection{Labeling Accuracy of ActiLabel}
    In this section, we report the labeling accuracy of DWMatching algorithm proposed in Section~\ref{subsec:optimal-label-learning} as the ratio of correctly labeled observations to all named \textit{labeling accuracy}. The labeling accuracy of the DWMatching algorithm mainly depends on the purity of the activity clusters and similarity between distribution of the data in the source and target.
    
    \subsubsection{Cross-Modality} 
    As shown in \figref{heatmap_mod}, accelerometer, gyroscope, magnetometer, and orientation modalities higher labeling accuracy (i.e.,  $70.2\%$--$88.0\%$) as the target sensor across all three datasets. In PAMAP2, the labeling accuracy drops to the range $45$\%--$0.65$\% when orientation and heart rate sensors are the target modality which shows the weak clustering of their observations into the activity classes and diverse data distribution from other modalities such as accelerometer. In Smartsock, DWMatching achieves $71.5\%$--$88.0\%$ labeling accuracy between an accelerometer and a stretch sensor.
     \begin{figure}[tbh!]
\centering
    \begin{subfigure}{0.75\columnwidth}
        \begin{subfigure}{0.95\columnwidth}
            \includegraphics[width=\linewidth]{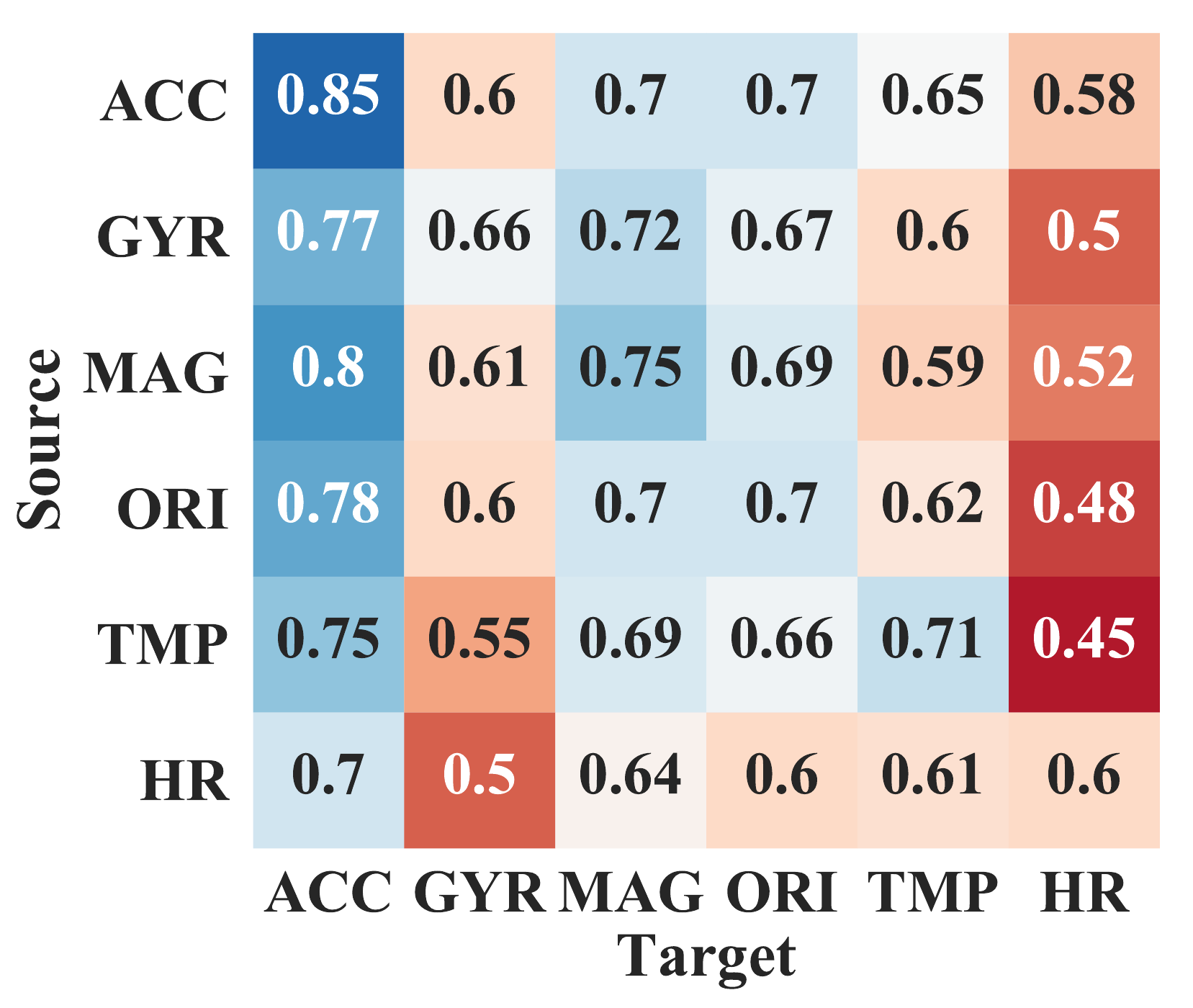}
            \caption{PAMAP2}
        \end{subfigure}%
        
         \begin{subfigure}{0.45\columnwidth}
            \includegraphics[width=\linewidth]{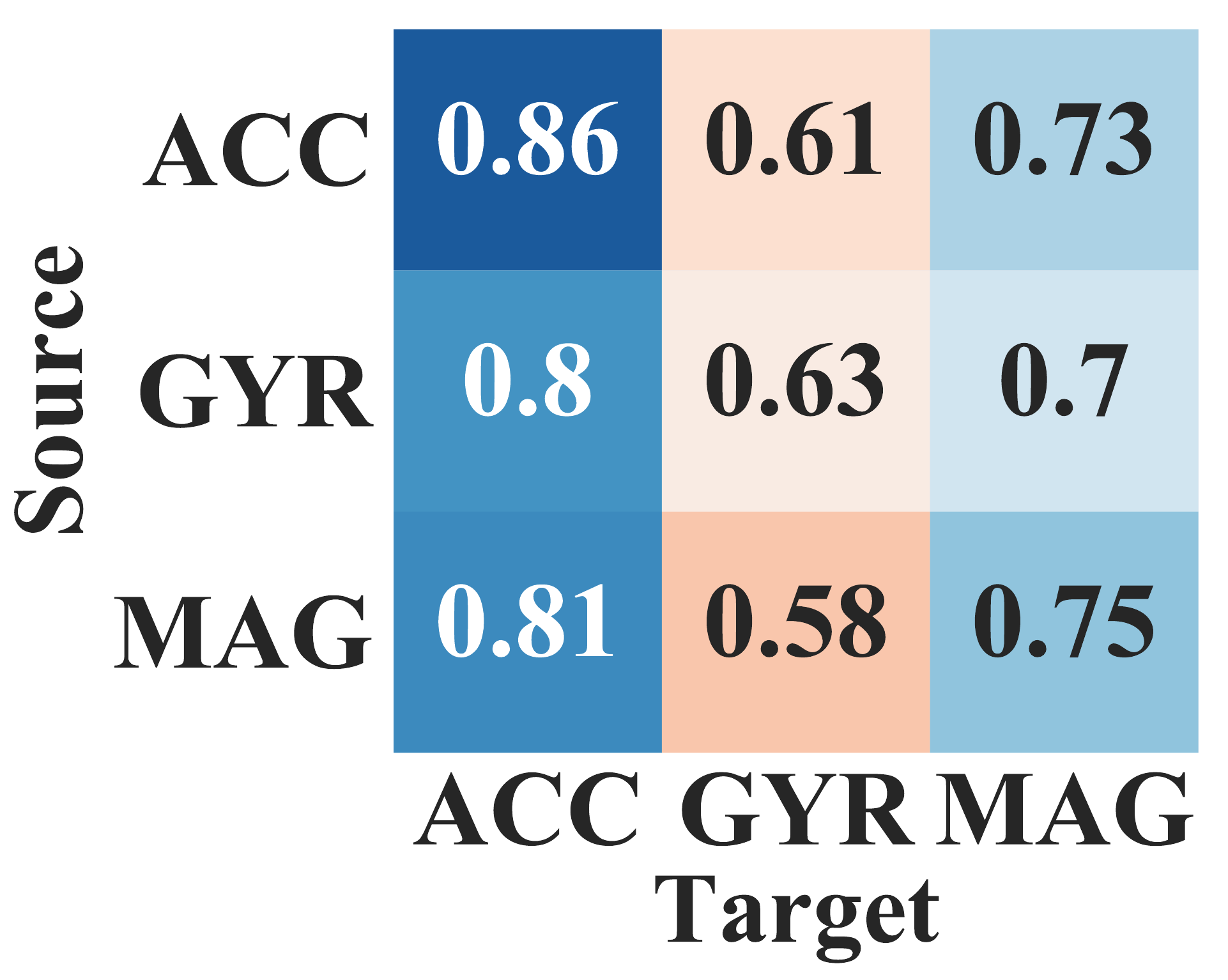}
            \caption{DAS}
        \end{subfigure}
        \begin{subfigure}{0.45\columnwidth}             \includegraphics[width=\linewidth]{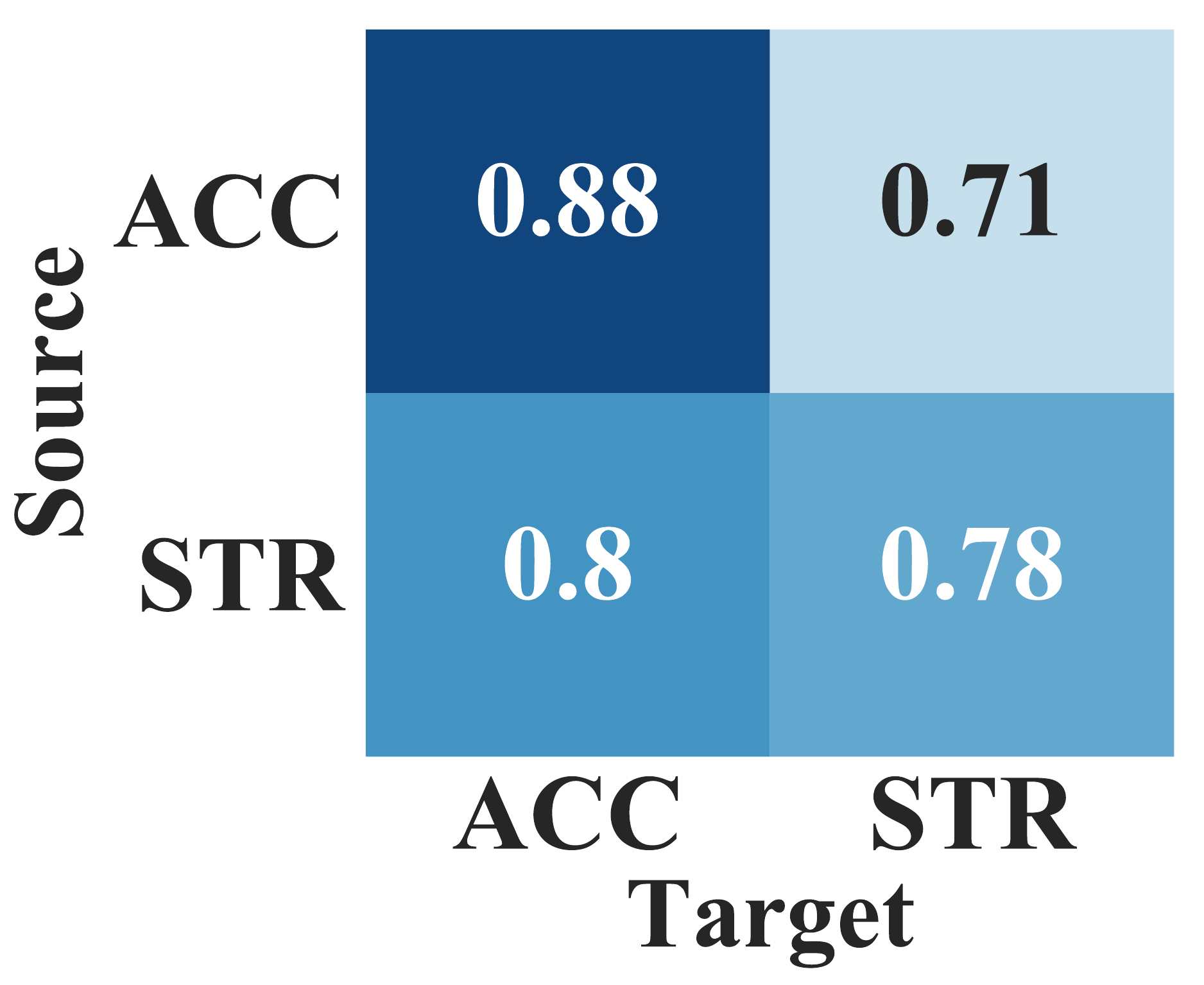}
            \caption{Smartsock}
            \label{fig:heatmap_loc}   
        \end{subfigure}
  \end{subfigure}
   \begin{subfigure}{0.12\columnwidth}
    \includegraphics[width=\linewidth]{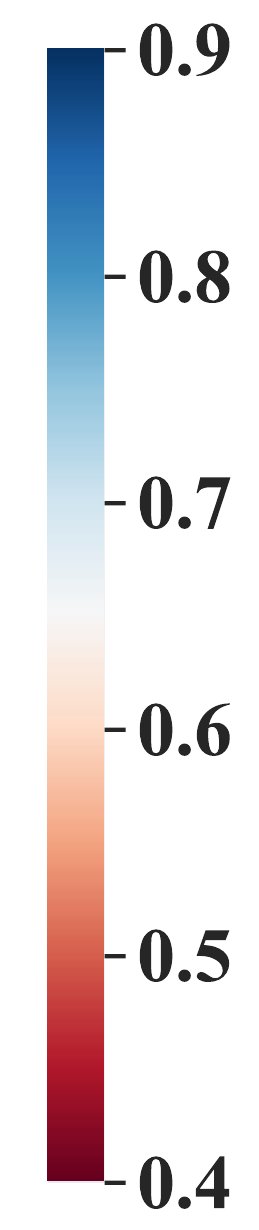}
  \end{subfigure}
  
    \caption{Labeling accuracy of ActiLabel for cross-modality scenario.}
\end{figure}

    \subsubsection{Cross-Location}
      Mappings between the same or similar body locations such as "chest to chest", and "left arm to right arm" achieve high labeling accuracies (i.e., $>70.3\%$). The labeling accuracy between dissimilar locations in the DAS dataset, such as "left leg to right arm" and "left arm to right leg", drops to the range $58.3\%$--$65.1\%$. Although chest, ankle, and hand are dissimilar body locations, mappings between them from the PAMAP2 dataset achieve $70.3\%$--$80.1\%$ labeling accuracy since data in each location comes from the rich collection of sensor modalities that provide sufficient information about inter-event structural similarities captured by our label learning algorithm. The cross-location transfer does not apply to the Smartsock dataset since it contained only one sensor location.
     \begin{figure}[tbh!]
\centering
        \begin{subfigure}{0.18\textwidth}
              \includegraphics[width=\linewidth]{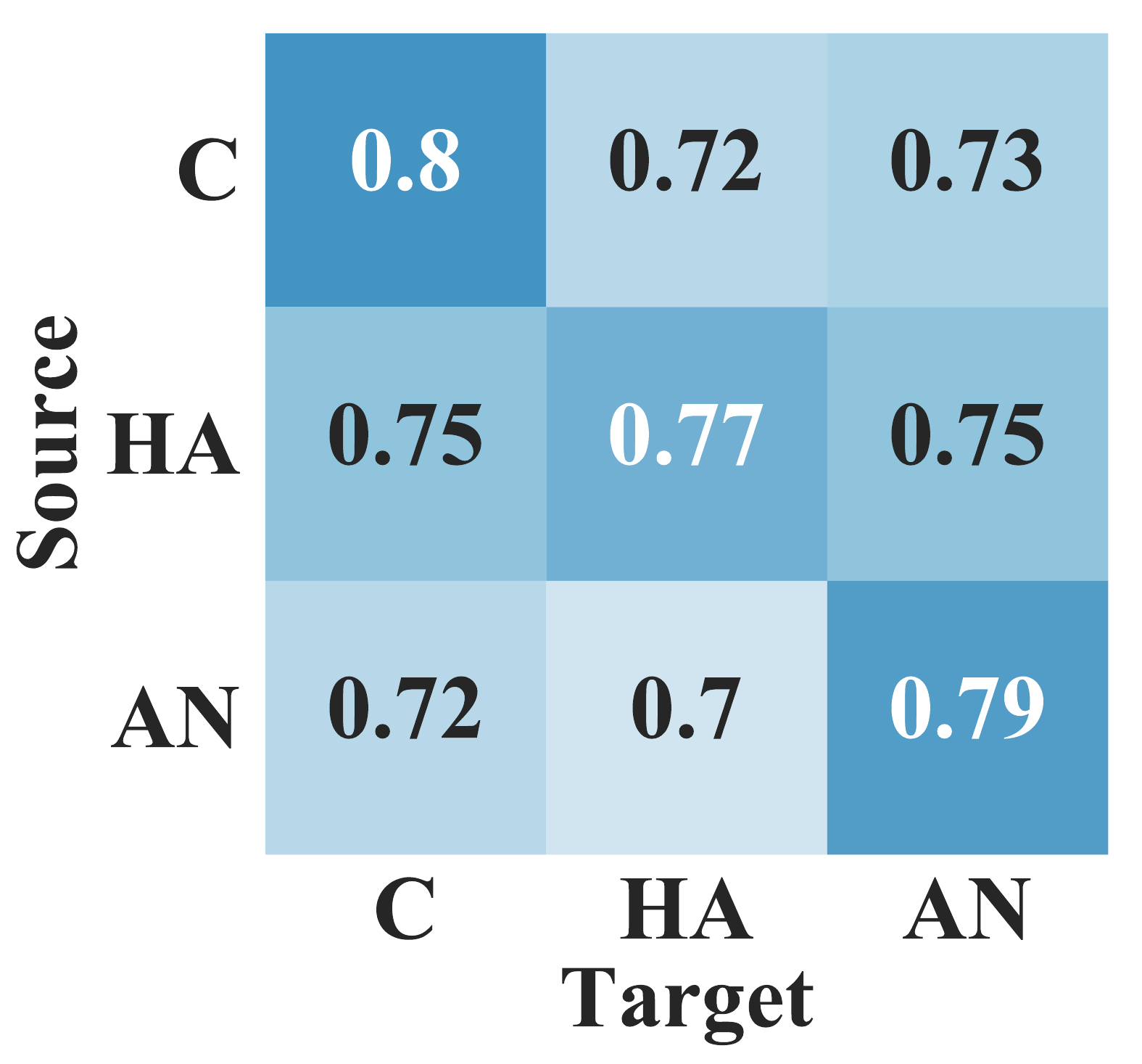}
              \caption{PAMAP2}
              \label{fig:heatmap_mod}
        \end{subfigure}%
         \begin{subfigure}{0.28\textwidth}
             \begin{subfigure}{\textwidth}
                  \includegraphics[width=\linewidth]{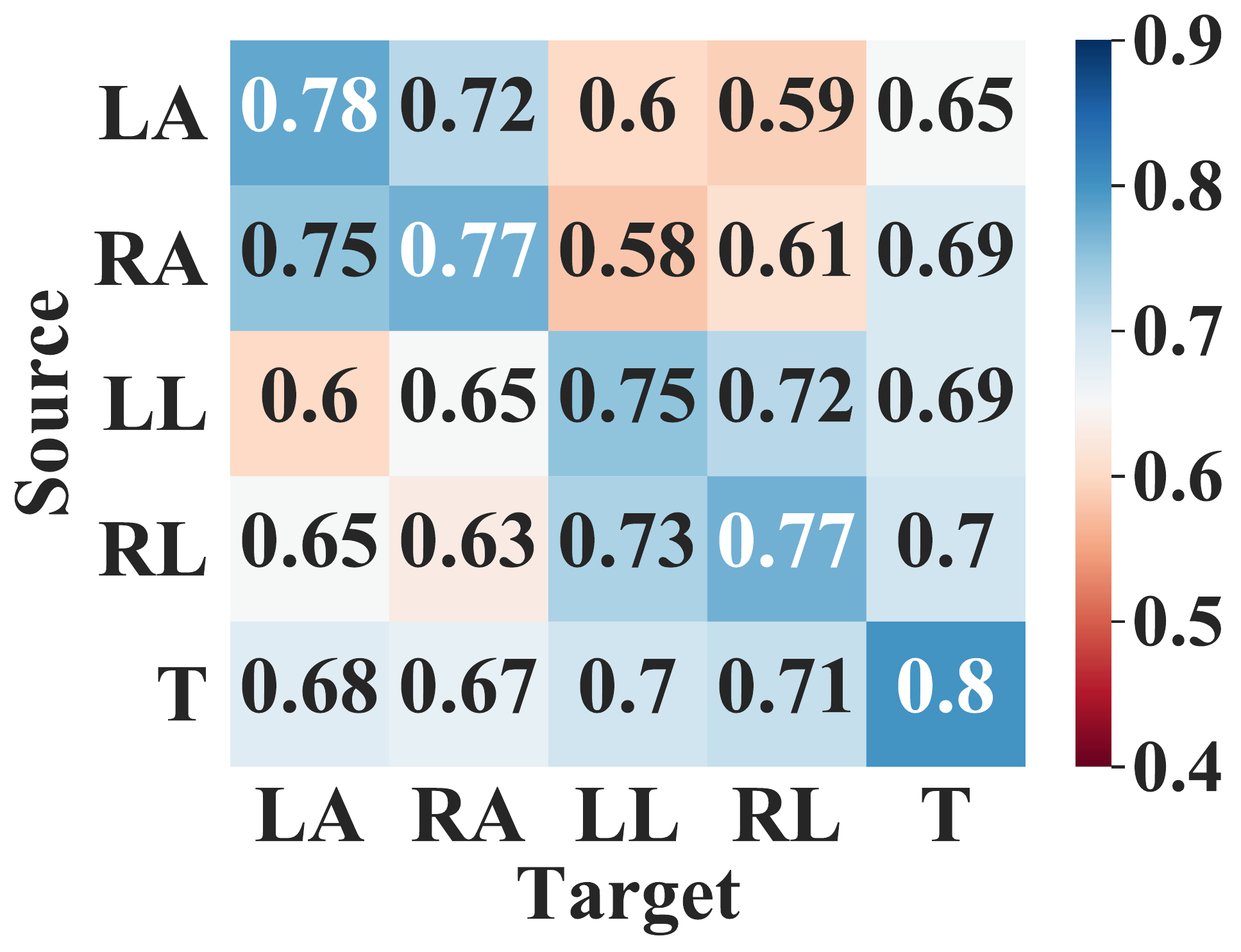}
                  \caption{DAS}
                  \label{fig:heatmap_loc}
                   \end{subfigure}
                   
        \end{subfigure}
    \caption{Labeling accuracy of ActiLabel for cross-location scenario.}
    \label{fig:labeling}
\end{figure}

\subsection{Performance of Activity Recognition}
    \tblref{activity} shows activity recognition performance for ActiLabel as well as algorithms under comparison, including baseline (BL), deep convolution LSTM (CL), DirectMap (DM), and upper-bound (UB) as discussed previously. We report the F1-Score value for each method as it is a better representative of the performance for unbalanced datasets.

    \subsubsection{Cross-Modality Transfer} \label{sec:map}
    We examined transfer learning across accelerometer, gyroscope, magnetometer, orientation, temperature, heart rate, and stretch sensor modalities. The cross-modality results in \tblref{activity} reflect average performance over all possible cross-modality scenarios. The baseline and ConvLSTM performed poorly overall three datasets, which shows the diverse distribution of data across sensors of different modalities. The DirectMap approach achieved $> 66.0\%$ F1-score over all three datasets. ActiLabel outperformed competing algorithms, in particular, DirectMap by $19.3\%$, $21.4\%$, and $6.7\%$ for PAMAP2, DAS, and Smartsock, respectively.
    
    \subsubsection{Cross-Location Transfer}
    We examined transfer learning among chest, ankle, hand, arms, legs, and torso locations. The cross-location results in \tblref{activity} represent average values over all possible transfer scenarios. The relatively low F1-scores of the baseline and ConvLSTM algorithms can be explained by the high level of diversity between the source and target domains during cross-location transfer learning. The DirectMap and ActiLabel both outperformed the baseline and ConvLSTM models. Specifically, DirectMap and ActiLabel $63.4\%$ and $70.8\%$ F1-Scores for PAMAP2, and $60.7\%$ and $68.4\%$ F1-Scores for DAS.

    \subsubsection{Cross-Subject Transfer}
    The DirectMap approach and ActiLabel obtained F1-Scores of $85.4\%$, and $82.7\%$ in PAMAP2, $77.59\%$ and $82.6\%$ in DAS, and $82.6\%$, and $77.5\%$ in Smartsock, respectively. Since there is a limited amount of data for each subject, ActiLabel could not capture high-level structures in the data. Therefore, it could not beat the state-of-the-art in all cases. All the algorithms achieved higher F1-score values compared to the cross-location and cross-modality scenarios. This observation suggests that data variations among different subjects can be normalized using techniques such as feature scaling, and feature selection before classification.

\begin{table}[t]
\small
\centering
\caption{Activity recognition performance (F1-Score).}
\begin{tabular}{llrrrrr}
\toprule
 \textbf{Scenario} &  \textbf{Dataset}    &  \textbf{BL}   &  \textbf{CL}  &  \textbf{DM} &  \textbf{AL} &  \textbf{UB} \\  \midrule
 
\multirow{3}{*}{\begin{tabular}[c]{@{}l@{}}\textbf{Cross-}\\  \textbf{modality}\end{tabular}}  &  \textbf{PAMAP2}     & 7.8   & 8.1   & 40.4 & \textbf{59.3} & \textcolor{darkgray}{80.8}  \\ 

&  \textbf{DAS}        & 9.3   & 8.2   & 44.8 &  \textbf{66.2}  & \textcolor{darkgray}{86.1}    \\ 

&  \textbf{Smartsock} & 16.2  & 12.8  & 66.0 &  \textbf{72.7} & \textcolor{darkgray}{84.2}   \\  \midrule

\multirow{2}{*}{\begin{tabular}[c]{@{}l@{}}\textbf{Cross-}\\  \textbf{location}\end{tabular}}    &  \textbf{PAMAP2}     & 14.3  & 12.7  & 63.4  &  \textbf{70.8}  & \textcolor{darkgray}{93.2} \\  

&  \textbf{DAS} & 13.2  & 12.4  & 60.7  &  \textbf{68.4}  & \textcolor{darkgray}{89.8}    \\  \midrule

\multicolumn{1}{l}{\multirow{3}{*}{\begin{tabular}[c]{@{}l@{}}\textbf{Cross-}\\  \textbf{location}\end{tabular}}} &  \textbf{PAMAP2}  & 65.8 & 61.9  & \textbf{85.4}  & 82.7  & \textcolor{darkgray}{98.1}    \\ 

\multicolumn{1}{l}{} &  \textbf{DAS}        & 67.1 & 56.8  & 79.0  &  \textbf{80.3}  & \textcolor{darkgray}{92.5}   \\  

\multicolumn{1}{l}{} &  \textbf{Smartsock} & 59.8  & 61.8  & \textbf{82.6}  & 80.0  & \textcolor{darkgray}{95.2}    \\  \midrule

\multicolumn{2}{c}{ \textbf{Average}}&{31.6} & 29.3 & 63.4 & \textbf{71.9}  &  \textcolor{darkgray}{89.9}    \\ \bottomrule
\end{tabular}
\label{tbl:activity}
\end{table}

    \section{Conclusion}
    We introduced ActiLabel, a computational framework with combinatorial optimization methodologies for transferring physical activity knowledge across highly diverse domains. ActiLabel extracts high-level structures from sensor observations in the target and source domains and learns labels in the target domain by finding an optimal mapping between dependency graphs in the source and target domains. ActiLabel provides consistently high accuracy for cross-domain knowledge transfer in various learning scenarios. Our extensive experimental results showed that ActiLabel achieves average F1-scores of $60.6\%$\%, $70.8$, and $82.7\%$ for cross-modality, cross-location, and cross-subject activity recognition, respectively. These results suggest that ActiLabel outperforms the competing algorithms by $36.3\%$, $32.7\%$, and $9.1\%$ in cross-modality, cross-location, and cross-subject learning, respectively.

\fontsize{9.3pt}{9.3pt} \selectfont
\bibliographystyle{named}
\bibliography{actilabel}

\end{document}